\documentclass{article} 
\usepackage[dvips]{graphicx}
\usepackage{amssymb,amsmath,color}
\usepackage{url}
\usepackage{algorithm}
\usepackage{algorithmic}

\title{Convex Relaxations for Learning Bounded Treewidth Decomposable Graphs}

\author{K. S. Sesh Kumar\\
INRIA-Sierra project-team\\ 
Laboratoire d{'}Informatique \\
de l{'}Ecole Normale Sup{\'e}rieure\\
Paris, France\\
\texttt{sesh-kumar.karri@inria.fr}\\
\and Francis Bach\\
INRIA-Sierra project-team\\ 
Laboratoire d{'}Informatique \\
de l{'}Ecole Normale Sup{\'e}rieure\\
Paris, France\\
\texttt{francis.bach@inria.fr}
}

\newcommand{\BEAS}{\begin{eqnarray*}}
\newcommand{\EEAS}{\end{eqnarray*}}
\newcommand{\BEA}{\begin{eqnarray}}
\newcommand{\EEA}{\end{eqnarray}}
\newcommand{\BEQ}{\begin{equation}}
\newcommand{\EEQ}{\end{equation}}
\newcommand{\BIT}{\begin{itemize}}
\newcommand{\EIT}{\end{itemize}}
\newcommand{\BNUM}{\begin{enumerate}}
\newcommand{\ENUM}{\end{enumerate}}
\newcommand{\BA}{\begin{array}}
\newcommand{\EA}{\end{array}}

\newcommand{\rb}{\mathbb{R}}
\newcommand{\BlackBox}{\rule{1.5ex}{1.5ex}}  

\newenvironment{proof}{\par\noindent{\bf Proof\ }}{\hfill\BlackBox\\[2mm]}

\newtheorem{proposition}{Proposition}

\newcommand{\mysec}[1]{Section~\ref{sec:#1}}
\newcommand{\myalg}[1]{Algorithm~\ref{alg:#1}}
\newcommand{\eq}[1]{Eq.~(\ref{eq:#1})}
\newcommand{\myfig}[1]{Figure~\ref{fig:#1}}

\begin{document} 
\maketitle

\begin{abstract} 
We consider the problem of learning the structure of undirected graphical models with bounded treewidth, within the maximum likelihood framework. This is an NP-hard problem and most approaches consider local search techniques. In this paper, we pose it as a combinatorial optimization problem, which is then relaxed to a convex optimization problem that involves searching over the forest and hyperforest polytopes with special structures, independently.
A supergradient method is used to solve the dual problem, with a run-time complexity of $O(k^3 n^{k+2} \log n)$ for each iteration, where $n$ is the number of variables and $k$ is a bound on the treewidth. We compare our approach to state-of-the-art methods on synthetic datasets and classical benchmarks, showing the gains of the novel convex approach.
\end{abstract} 

\section{Introduction}

Graphical models provide a versatile set of tools for probabilistic modeling of large collections of interdependent variables. They are defined by graphs that encode the conditional independences among the random variables, together with  potential functions or conditional probability distributions that encode the specific local interactions leading to globally well-defined probability distributions~\cite{bishop2006pattern,wainwright2008graphical,koller2009probabilistic}.

In many domains such as computer vision, natural language processing or bioinformatics, the structure of the graph follows naturally from the constraints of the problem at hand. In other situations, it might be desirable to estimate this structure from a set of observations. It allows (a) a statistical fit of rich probability distributions that can be considered for further use, and (b) discovery of structural relationship between different variables. In the former case, distributions with tractable inference are often desirable, i.e., inference with run-time complexity does not scale exponentially in the number of  variables in the model. The simplest constraint to ensure tractability is to impose tree-structured graphs~\cite{Chow68approximatingdiscrete}. However, these distributions are not rich enough, and following earlier work~\cite{malvestuto1991approximating,bach2002thin,narasimhan2004pac,Chechetka:2007,GogateWD10,t-cherry:2011}, we consider models with bounded \emph{treewidth}, not simply by one (i.e., trees), but by a small constant~$k$.

Beyond the possibility of fitting tractable distributions (for which probabilistic inference has linear complexity in the number of variables),  learning bounded-treewidth graphical models is a key to design approximate inference algorithms for graphs with higher treewidth. Indeed, as shown by~\cite{Saul95,wainwright2008graphical,kolomogrov2012}, approximating general distributions by tractable distributions is a common tool in variational inference. However, in practice, the complexity of   variational distributions is often limited to trees (i.e., $k=1$), since these are the only ones with exact polynomial-time structure learning algorithms. The convex relaxation designed in this paper enables us to augment the applicability of variational inference, by allowing a finer trade-off between run-time complexity and approximation quality.

Apart from trees, learning the structure of a directed or undirected graphical model, with or without constraints on the treewidth, remains a hard problem. Two types of algorithms have emerged, based on the two equivalent definitions of graphical models: (a) by testing conditional independence relationships~\cite{spirtes2001causation} or (b) by maximizing the log-likelihood of the data using the factorized form of the  distribution~\cite{friedman2003being}. 
In the specific context of learning bounded-treewidth graphical models, the latter approach has been shown to be NP-hard~\cite{srebro2002maximum} and led to various approximate algorithms based on local search techniques~\cite{malvestuto1991approximating,deshpande2001efficient,Karger:2001, bach2002thin, Shahaf:2009, t-cherry:2011} while the former approach led to  algorithms based on independence tests~\cite{narasimhan2004pac, Chechetka:2007, GogateWD10}, which have recovery guarantees when the data-generating distribution has low treewidth. Malvestuto~\cite{malvestuto1991approximating} proposed a greedy heuristic of hyperedge selection with least incremental entropy. Deshpande et al.~\cite{deshpande2001efficient} proposed a simple edge selection technique that maintains decomposability of the graph while minimizing the KL-divergence to the original distribution. Karger et al.~\cite{Karger:2001} proposed the first convex optimization approach to learn the maximum weighted $k$-windmill, a sub-class of the decomposable graph. Bach et al.~\cite{bach2002thin} gave an approach which iteratively refines the hyperedge selection based on KL-divergence using iterative scaling. Shahaf et al.~\cite{Shahaf:2009} proposed another convex optimization approach with Bethe approximation of the likelihood using graph-cuts. Sz{\'a}ntai et al.~\cite{t-cherry:2011} proposed a hyperedge selection criteria based on high mutual information within a hyperedge. Narasimhan et al.~\cite{narasimhan2004pac} performs independence tests by solving submodular optimization problems and derives a decomposable graph using dynamic programming. Chechetka et al.~\cite{Chechetka:2007} used the weaker notion of conditional mutual information instead of conditional independence to learn approximate junction trees. Gogate et al.~\cite{GogateWD10} uses low mutual information criteria to recursively split the state space to smaller subsets until no further splits are possible.

In this paper, we make the following contributions:
\begin{itemize}
\item We provide a novel convex relaxation for learning bounded-treewidth decomposable graphical models from data in polynomial time. This is achieved by posing  the problem as a combinatorial optimization problem in \mysec{MLDG}, which is relaxed to a convex optimization problem that involves the graphic and hypergraphic matroids, as shown in \mysec{relax}.
\item We show in \mysec{LD} how a supergradient ascent method may be used to solve the dual optimization problem, using greedy algorithms as inner loops on the two matroids. Each iteration has a run-time complexity of $O(k^3 n^{k+2} \log n)$, where $n$ is the number of variables. We also show how to round the obtained fractional solution.
\item We compare our approach to state-of-the-art methods on synthetic datasets and classical benchmarks in \mysec{results}, showing the gains of the novel convex approach.
\end{itemize}

\begin{figure}[t]
\begin{center}
\begin{tabular}{ccc}
\includegraphics[width=0.2\textwidth]{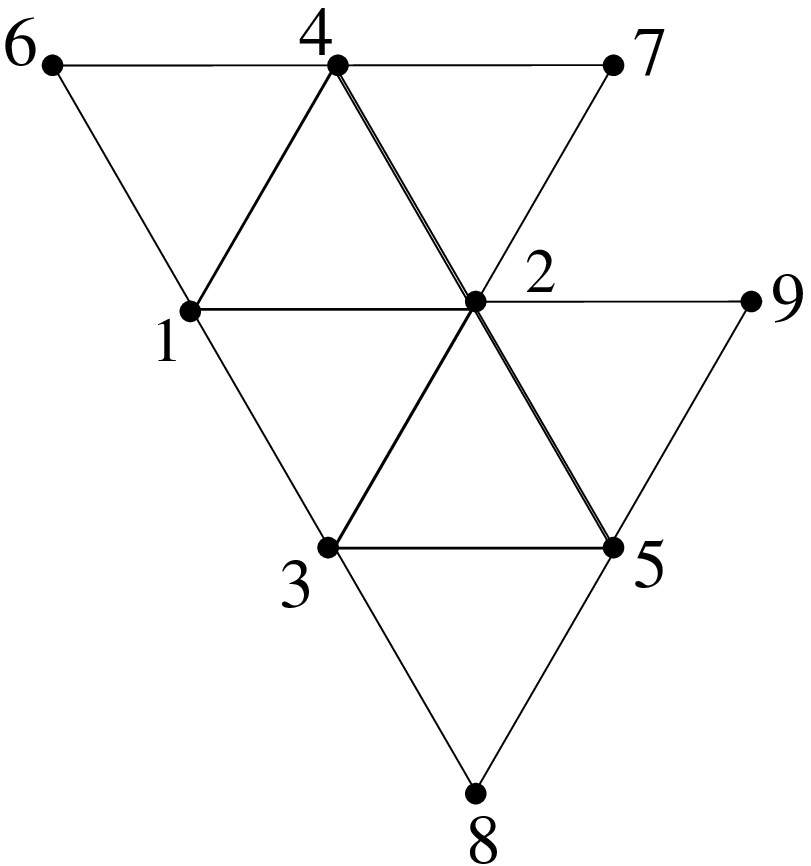}&
\includegraphics[width=0.2\textwidth]{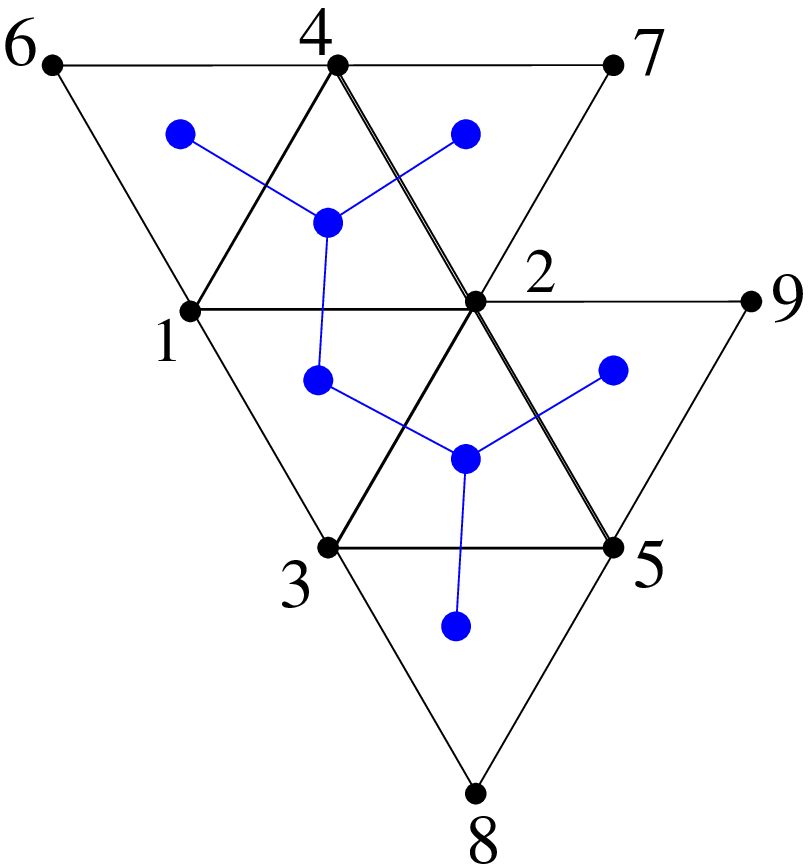}&
\includegraphics[width=0.5\textwidth]{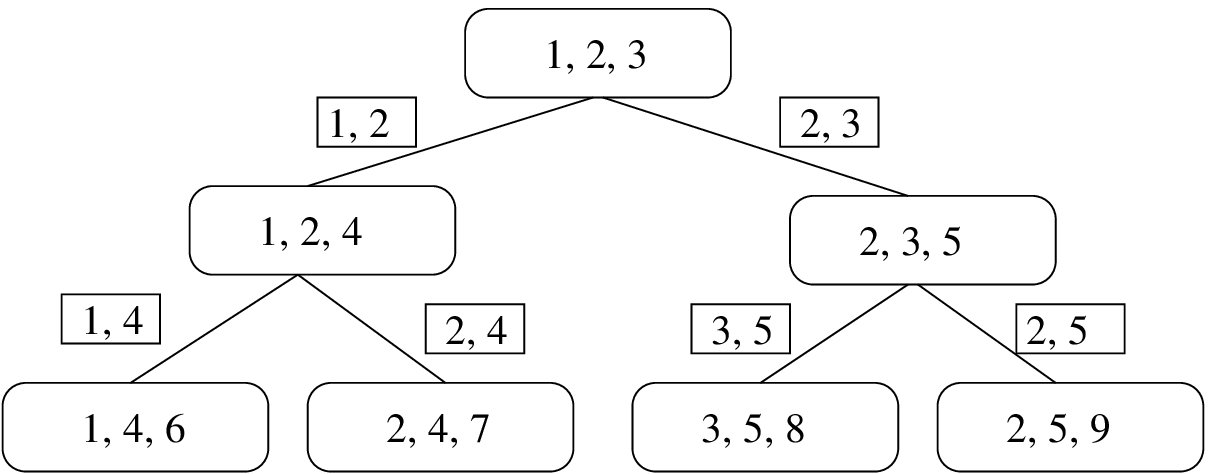}\\
(a) & (b) & (c)
\end{tabular}
\end{center}

\caption{(a) A decomposable graph on the set of vertices $V=\{1, 2, 3, 4, 5, 6, 7, 8, 9\}$ having treewidth 2.(b) A junction tree embedded on the decomposable graph representing the maximal cliques by blue dots and the separator sets by blue lines. (c) The corresponding junction tree representation of the decomposable graph with ovals representing the maximal cliques and the rectangles representing the corresponding separator set.}
\label{fig:sampleGraph}
\end{figure}

\section{Maximum Likelihood Decomposable Graphical Models}
\label{sec:MLDG}

In this section, we first review the relevant concepts of decomposable graphs and junction trees; for more details, see~\cite{bishop2006pattern,wainwright2008graphical,koller2009probabilistic}. We then cast the problem of learning the maximum likelihood bounded treewidth graph as a combinatorial optimization problem.

\subsection{Decomposable graphs and junction trees}
 We assume we are given an \emph{undirected} graph $G$ defined on the set of vertices $V = \{1, 2, \ldots, n\}$. Let $\mathcal{C}(G)$ denote the set of maximal cliques of $G$ (which we will refer to as cliques). We consider $n$ random variables $X_1,\dots,X_n$ (referred to as $X$), associated with each vertex indexed by $V$. For simplicity, they are assumed to be discrete, but this is not a restriction as  maximum likelihood will use only entropies that can be extended to differentiable entropies~\cite{cover2006elements}.
 
The distribution $p(x)$ of $X$ is said to factorize in the graph $G$, if and only if it factorizes as a product of potentials that depend only on the variables within maximal cliques.
 
A graph is said to be \emph{decomposable} if it has a junction tree, i.e., a  {spanning} tree whose vertices are maximal cliques of $G$ (i.e., $\mathcal{C}(G)$ is the vertex set) such that:
 
\begin{itemize}

\item the junction tree connects only cliques that have a common element (\emph{clique tree} property),
\item for any vertex $i \in V$, the subgraph of cliques containing $i$ is a tree (\emph{running intersection} property).
\end{itemize}

Let $\mathcal{T}(G)$ denote the edges of the junction tree over the set of 
cliques $\mathcal{C}(G)$. When the graph $G$ is decomposable, the distribution $p(x)$ of $X$ factorizes in $G$ if and only if it may be written as
\begin{equation}
p_G(x) = \frac{\prod_{C \in \mathcal{C}(G)} p_C(x_C)}{\prod_{(C,D) \in \mathcal{T}(G)} p_{C \cap D}(x_{C \cap D})} \text{ ,}
\label{eq:jpd}
\end{equation}
 where $x$ is an instance in the domain of $X$, which we denote by $\mathcal{X}$. $p_C(x_C)$ denotes the marginal distribution of random 
variables belonging to $C \in \mathcal{C}(G)$ and $p_{C \cap D}(x_{C \cap D})$ denotes the marginal distribution 
of random variables belonging to the {\em separator} set $C \cap D$, such that $(C,D) \in \mathcal{T}(G)$. See \myfig{sampleGraph}. The \emph{treewidth} of $G$ is the maximal size of the cliques in $G$, minus one.

An alternative representation of decomposable graphs may be obtained by considering \emph{hypergraphs}. Hypergraphs are defined by a base set $V$ and  a set of hyperedges, i.e., subsets of $V$. A hypergraph is said to be acyclic if and only if the resulting graph obtained by connecting all elements within an hyperedge is  decomposable. Unfortunately, the nice properties of acyclic graphs do not transfer to acyclic hypergraphs. Particulary, the matroid property, which allows exact greedy algorithms, does not hold. 
In \mysec{matroid2}, we will use a lesser known more general notion of acyclicity that will lead to exact greedy algorithms.

\subsection{Maximum likelihood estimation}
Given $N$ observations $x^1,\dots,x^N$  of $X$, we denote the corresponding empirical distribution of $X$ by $\hat{p}(x) = \frac{1}{N} \sum_{i=1}^N \delta( x = x^i)$. Given the structure of a decomposable graph $G$, the maximum likelihood distribution that factorizes in $G$ may be obtained by combining the marginal empirical distributions on all maximum cliques and their separators.

Let $\hat{p}(x)$ denote the empirical distribution and $\hat{p}_G(x)$ denotes the projected distribution on a decomposable graph $G$. 
Estimating the maximum likelihood decomposable graph which best approximates $\hat{p}$ is equivalent to finding the graph, $G$, which minimizes 
the KL-divergence between the target distribution and the projected distribution, $\hat{p}_G$, defined by $D(\hat{p}||\hat{p}_G)$. 
\begin{eqnarray}
D(\hat{p}||\hat{p}_G) & = & \sum_{x \in \mathcal{X}} \hat{p}(x) \log\frac{\hat{p}(x)}{\hat{p}_G(x)}\nonumber \\
          & \propto & \sum_{x \in \mathcal{X}}- \hat{p}(x) \log \hat{p}_G(x) \text{  as $\hat{p}(x)$ is independent of $G$ }\nonumber  \\
          & = & \sum_{x \in \mathcal{X}}- \hat{p}(x) \log  \frac{\prod_{C \in \mathcal{C}(G)} \hat{p}_C(x_C)}{\prod_{(C,D) \in \mathcal{T}(G)} \hat{p}_{C \cap D}(x_{C \cap D})} \text{ from \eq{jpd} }\nonumber\\
          & = & \sum_{x \in \mathcal{X}} \bigg(- \hat{p}(x) \log\prod_{C \in \mathcal{C}(G)} \hat{p}_C(x_C)\bigg)- \sum_{x \in X} \bigg(- \hat{p}(x)\log \prod_{(C,D) \in \mathcal{T}(G)} \hat{p}_{C \cap D}(x_{C \cap D})\bigg) \nonumber \\
          & = & \sum_{C \in \mathcal{C}(G)} \sum_{x \in \mathcal{X}} - \hat{p}(x) \log \hat{p}_C(x_C)- \sum_{(C,D) \in \mathcal{T}(G)} \sum_{x \in \mathcal{X}} - \hat{p}(x) \log \hat{p}_{C \cap D}(x_{C \cap D}) \nonumber \\
          & = & \sum_{C \in \mathcal{C}(G)} \sum_{x_C \in \mathcal{X}_C} - \hat{p}_C(x_C) \log \hat{p}_C(x_C)- \sum_{(C,D) \in \mathcal{T}(G)} \sum_{x_{C \cap D} \in \mathcal{X}_{C \cap D}} - \hat{p}_{C \cap D}(x_{C \cap D}) \log \hat{p}_{C \cap D}(x_{C \cap D}) \nonumber \\
          & = & \sum_{C \in \mathcal{C}(G)} H(C)- \sum_{(C,D) \in \mathcal{T}(G)} H({C \cap D}) \label{eq:GML} \text{ ,}
\end{eqnarray}
where $H(S)$ is the empirical entropy of the random variables indexed by the set $S \subseteq V$, defined by 
$H(S) = \sum_{x_S} \{ - \hat{p}_S(x_S) \log \hat{p}_S(x_S) \}$, and where the sum is taken over all possible values of $x_S$.

Note that in this paper, we will not be using a traditional model selection term~\cite{friedman2003being} as we will only consider models of low tree-width (with a bounded number of parameters).

\subsection{Combinatorial optimization problem}
\label{sec:comb}
We now consider  the problem of learning a decomposable graph of treewidth less than $k$. We assume that we are given all entropies $H(S)$ for  subsets $S$ of $V$ of cardinality less than $k+1$.

Since we do not add any model selection term, without loss of generality~\cite{SzantaiK12}, we restrict the search space to the space of \emph{maximal junction trees}, i.e., junction trees with $n-k$ maximal cliques of size $k+1$ and $n-k-1$ separator sets of  size $k$ between two cliques of size $k+1$. Our natural search spaces are thus characterized by
  $\mathcal{D}$, the set of all subsets of size $k+1$ of $V$, of cardinality ${n \choose k+1}$, and $\mathcal{E}$, the set of all potential edges in a junction tree, i.e., 
$\mathcal{E} = \{(C,D) \in \mathcal{D} \times \mathcal{D}, C \cap D \neq \varnothing, |C \cap D| = k\}$.
The cardinality of $\mathcal{E}$ is ${n \choose k+2}.{k+2 \choose 2}$ (number of subsets of size $k+2$ times the number of possibility of excluding  two elements to obtain a separator). 

A decomposable graph will be represented by a clique selection function $\tau: \mathcal{D} \to \{0,1\}$ and an edge selection function $\rho: \mathcal{E} \to \{0,1\}$ so that $\tau(C)=1$ if $C$ is a maximal clique of the graph and $\rho(C,D)=1$ if $(C,D)$ is an edge in the junction tree. Both $\rho$ and $\tau$ will be referred to as incidence \emph{functions} or incidence \emph{vectors}, when seen as elements of $\{0,1\}^\mathcal{D}$ and $\{0,1\}^\mathcal{E}$.

Thus, minimizing the problem defined in \eq{GML} is equivalent to minimizing,
\begin{equation}
\mathcal{P}(\tau, \rho)\! =\!  \sum_{C \in \mathcal{D}} H(C)\tau(C) -\!\! \!\! \sum_{(C,D) \in \mathcal{E}}\!\!\!\! H(C \cap D) \rho(C,D), \label{eq:primalCost}
\end{equation}
with the constraint that $(\tau, \rho)$ forms a decomposable graph. 

At this time, we have merely reparameterized the problem with the clique and edge selection functions. We now consider a set of necessary and sufficient conditions for the pair to form a decomposable graph. Some are convex in $(\tau,\rho)$, while some are not. The latter ones will be relaxed in \mysec{relax}.
From now on,  we denote by $1_{i \in C}$ the
 indicator function for $i \in C$ (i.e., it is equal to 1 if $i \in C$ and zero otherwise).

\begin{itemize}
\item{\em Covering $V$}: Each vertex in $V$ must be covered by atleast one of the selected cliques,
\begin{equation}
\forall i \in V, \ 
\sum_{C \in \mathcal{D}} 1_{i \in C} \tau(C) \geq 1.
\label{eq:constraint1}
\end{equation}

\item{\em Number of edges}: Exactly $n-k-1$ edges from $\mathcal{E}$ must be selected,
\begin{equation}
\sum_{(C,D) \in \mathcal{E}} \rho(C,D) = n-k-1.
\label{eq:constraint6}
\end{equation}

\item{\em Number of cliques}: Exactly $n-k$ cliques from $\mathcal{D}$ must be selected, 
\begin{equation}
\sum_{C \in \mathcal{D}} \tau(C) = n-k.
\label{eq:constraint7}
\end{equation}

\item{\em Running intersection property}: Every vertex, $i \in V$ must induce a tree, i.e., the number of selected edges containing the vertex, $i$, must be equal to the number of selected cliques containing the vertex, $i$ , minus one. 
\begin{equation}
\!\!\!\!\!\!\forall i \in V, \sum_{(C,D) \in \mathcal{E}} \!\!1_{i \in (C \cap D)} \rho(C,D) - \sum_{C \in \mathcal{D}} 1_{i \in C} \tau(C) +  1 = 0.
\label{eq:constraint3}
\end{equation}

\item{\em Edges between selected cliques}: An edge in $\mathcal{E}$ is selected by $\rho$ only if the cliques it is incident on is selected by $\tau$. 
\begin{equation}
\forall C \in \mathcal{D}, \  \tau(C) = \max_{ D \in \mathcal{D},  \    (C,D) \in \mathcal{E} }\rho(C,D).
\label{eq:constraint444}
\end{equation}

\item{\em Acyclicity of $\rho$}: $\rho$ selects edges in $\mathcal{E}$ such that they do not have loops, e.g., the blue lines in \myfig{sampleGraph}-(b) cannot form loops,
\begin{equation}
\text{$\rho$ represents a subforest of the graph $(\mathcal{D}, \mathcal{E})$.}
\label{eq:constraint2}
\end{equation}

\item{\em Acyclicity of $\tau$}: $\tau$ selects the hyperedges of $V$ in $\mathcal{D}$ such that they are acyclic, i.e.,
\begin{equation}
\text{$\tau$ represents an acyclic hypergraph of $(V, \mathcal{D})$.}
\label{eq:constraintww2}
\end{equation}

\end{itemize}

The above constraints encode the classical definition of junction trees. Thus our combinatorial problem is exactly equivalent to
minimizing $P(\tau,\rho)$ defined in \eq{primalCost}, subject to the constraints in
in  \eq{constraint1},  \eq{constraint6}, \eq{constraint7}, \eq{constraint3}, \eq{constraint444}, \eq{constraint2} and \eq{constraintww2}. Note that the constraint  \eq{constraintww2} that $\tau$ represents an acyclic hypergraph is implied by the other constraints.

\myfig{sampleSpace} shows clique and edge selections in blue which satisfy all these constraints and hence represent a decomposable graph. The clique and edge 
selections in red violates at least one of these constraints.

\begin{figure}[t]
\begin{center}
\includegraphics[width=0.45\textwidth]{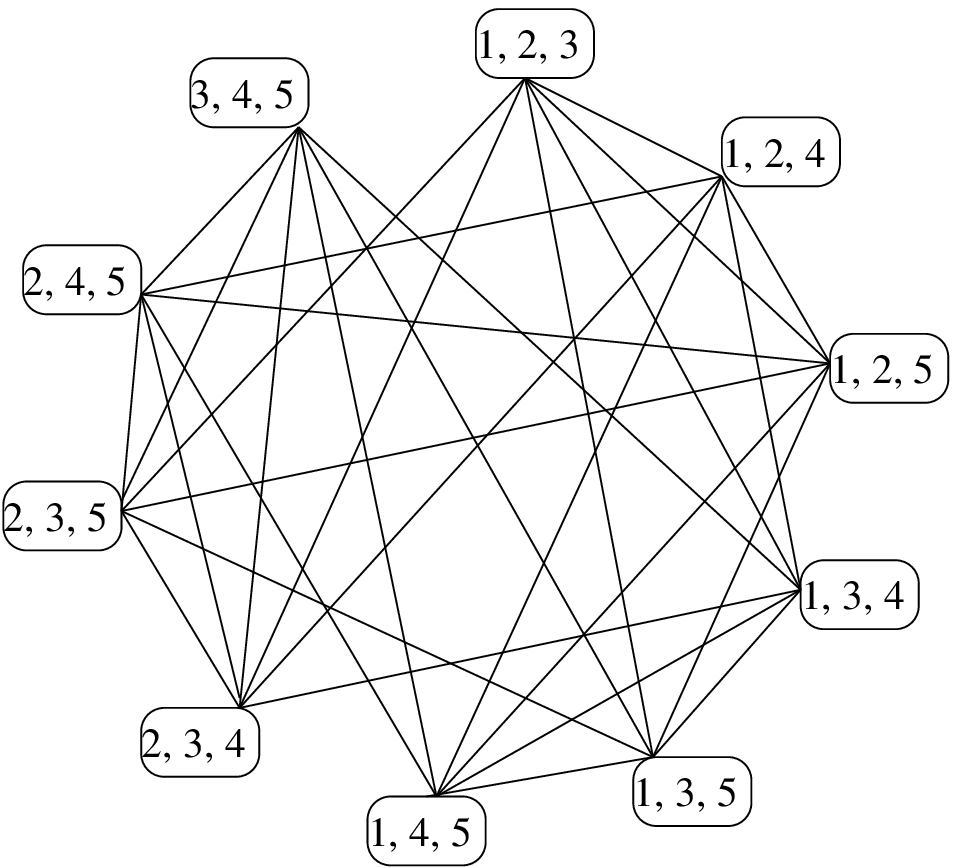}
\includegraphics[width=0.45\textwidth]{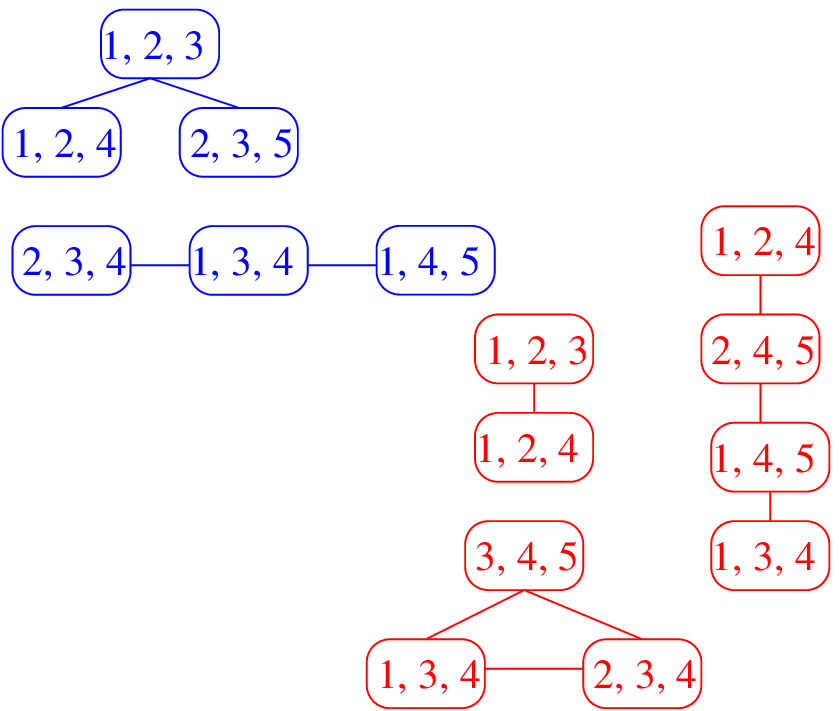}
\end{center}
\caption{Space of cliques $\mathcal{D}$ denoted by ovals and the space of feasible edges $\mathcal{E}$ denoted by lines 
for $V=\{1, 2, 3, 4, 5\}$ and treewidth 2(in Black). Clique and edge selections in blue represent decomposable graphs while those 
in red denote graphs that are not decomposable (best seen in color).}
\label{fig:sampleSpace}
\end{figure}

\section{Convex Relaxation}
\label{sec:relax}

We now provide a convex relaxation of the combinatorial problem defined in \mysec{comb}. The covering constraint in \eq{constraint1}, the number of edges and the number of cliques constraints in \eq{constraint6} and \eq{constraint7} respectively, and the running intersection property in \eq{constraint3} are already convex in~$(\tau,\rho)$.

The constraint in \eq{constraint444} that $\forall C \in \mathcal{D}, \  \tau(C) = \max_{ D \in \mathcal{D}, \    (C,D) \in \mathcal{E} }\rho(C,D)$ may be relaxed into:

\begin{itemize}
\item {\em Edge constraint}: selection of edges only if the both the incident cliques are selected, i.e.,
\begin{equation}
\forall C \in \mathcal{D},\  \forall (C,D) \in \mathcal{E}, \ \rho(C,D) \leq \tau(C).
\label{eq:constraint4}
\end{equation}

\item{\em Clique constraint}: selection of a clique if at least an edge incident on it is selected, i.e.,
\begin{equation}
\forall C \in \mathcal{D}, \ \tau(C) \leq \sum_{(C,D)  \in \mathcal{E}} \rho(C,D).
\label{eq:constraint5}
\end{equation}
\end{itemize}

We now consider the two acyclicity constraints in \eq{constraint2} and \eq{constraintww2}.

\subsection{Forest polytope}
\label{sec:matroid1}

Given the graph $(\mathcal{D}, \mathcal{E})$, the \emph{forest polytope} is the convex hull of all incidence vectors $\rho$ of subforests of $(\mathcal{D}, \mathcal{E})$. Thus, it is exactly the convex hull of all $\rho: \mathcal{E} \to \{0,1\}$ such that $\rho$ satisfies the constraint in \eq{constraint2}. We may thus relax it into:

\begin{list}{\labelitemi}{\leftmargin=1.1em}
   \addtolength{\itemsep}{-.45\baselineskip}
   
\item[--]
{\em Tree constraint}:

\begin{equation}
\text{$\rho$ is in the forest polytope of $(\mathcal{D}, \mathcal{E})$.}
\label{eq:constraint2C}
\end{equation}
\end{list}

While the new constraint in \eq{constraint2C} forms a convex constraint, it is crucial that it may be dealt with empirically in polynomial time. This is made possible by the fact that one may maximize any linear function over that polytope. Indeed, for a weight function $w: \mathcal{E} \times \mathcal{E} \to \rb$, maximizing $\sum_{(C,D) \in \mathcal{E}} w(C,D) \rho(C,D)$ is exactly a maximum weight spanning forest problem, and its solution may be obtained by Kruskal's algorithm, i.e., (a) order all (potentially negative) weights $w(C,D)$ and (b) greedily select edges $(C,D)$, i.e., set $\rho(C,D)=1$, with higher weights first, as long as they form a forest and as long as the weights are positive. When we add the restriction that the number of edges is fixed (in our case $n-k-1$), then the algorithm is stopped when exactly the desired number of edges is selected (whether the corresponding weights are positive or not). See, e.g.,~\cite{schrijver2004combinatorial}.

The polytope defined above may also be defined as the independence polytope of the graphic matroid, which is the traditional reason why the greedy algorithm is exact~\cite{schrijver2004combinatorial}. In the next section, we show how this can be extended to hypergraphs.

\subsection{Hypergraphic matroid}
\label{sec:matroid2}

Given the set of potential cliques $\mathcal{D}$ over $V$, we consider functions $\tau: \mathcal{D} \to \{0,1\}$ that are equal to one when a clique is selected, and zero otherwise. Ideally, we would like to treat the acyclicity of the associated hypergraph in a similar way than for regular graphs.
However, the set of acyclic subgraphs of the hypergraph defined from $\mathcal{D}$ does not form a matroid, and thus the polytope defined as the convex hull of all incidence vectors/functions of acyclic hypergraphs may be defined, but 
the greedy algorithm is not applicable. In order to define what is referred to as the \emph{hypergraphic matroid}, one needs to relax the notion of acyclicity.

We now follow~\cite{lorea1975hypergraphes,frank2003decomposing,fukunaga2010computing} and define a different notion of acyclicity for hypergraphs. An hypergraph $(V,\mathcal{F})$ is an \emph{hyperforest} if and only if
 for all $A \subset V$, the number of hyperedges in $\mathcal{F}$ contained in $A$ is less than $|A|-1$. A non-trivially equivalent definition is that we can select two elements in each hyperedege so that the graph with vertex set $V$ and with edge set composed of these pairs  is a forest.

Given an hypergraph with hyperedge set $\mathcal{D}$, the set of sub-hypergraphs which are hyperforests forms a matroid. This implies that given a weight function on $\mathcal{D}$, one may find the maximum weight hyperforest with a greedy algorithm that ranks all hyperedges and select them as long as they don't violate acyclicity (with the notion of acyclicity just defined and for which we exhibit a test below).

\paragraph{Testing acyclicity.} 
Checking acyclicity of an hypergraph $(V,\mathcal{F})$ (which is needed for the greedy algorithm above) may be done by minimizing with respect to $ A \subset V$
$$
|A| - \sum_{ G \in \mathcal{F}} 1_{ G \subset A}.
$$
The hypergraph is an hyperforest if and only if the minimum is greater or equal to one. The minimization of this function may be cast a min-cut/max-flow problem as follows~\cite{fukunaga2010computing}:

\begin{list}{\labelitemi}{\leftmargin=1.1em}
   \addtolength{\itemsep}{-.45\baselineskip}

\item[--] single source, single sink, one node per hyperedge in $\mathcal{F}$, one node per vertex in $V$,
\item[--] the source points towards each hyperedge with unit capacity,
\item[--] each hyperedge points towards the vertices it contains, with infinite capacity,
\item[--] each vertex points towards the sink, with unit capacity.
\end{list}
 
\paragraph{Link with decomposability.} 
 
 The hypergraph obtained from the maximal cliques of a decomposable graph can easily be seen to be an hyperforest. But the converse is not true.

\paragraph{Hyperforest polytope.}
We can now naturally define the hyperforest polytope as the convex hull of all incidence vectors of hyperforests. Thus the constraint in \eq{constraintww2} may be relaxed into:

\begin{list}{\labelitemi}{\leftmargin=1.1em}
   \addtolength{\itemsep}{-.45\baselineskip}

\item[--]
{\em Hyperforest constraint}:

\begin{equation}
\text{$\tau$ is in the hyperforest polytope of $(V, \mathcal{D})$.}
\label{eq:constraintww2C}
\end{equation}
\end{list}

\subsection{Relaxed optimization problem}

We can now formulate our combinatorial problem from the constraints in \eq{constraint1}, \eq{constraint6}, \eq{constraint7}, \eq{constraint3}, \eq{constraint4}, \eq{constraint5}, \eq{constraint2C} and \eq{constraintww2C} as follows
\begin{equation}
    \min \mathcal{P}(\tau, \rho)\text{ subject to }  \left\{
    \begin{array}{l}
    \tau \in \{0,1\}^{\mathcal{D}},\\
    \rho \in \{0,1\}^{\mathcal{E}},\\
    \forall i \in V, \sum_{C \in \mathcal{D}}1_{i \in C} \tau(C) \geq 1, \\
    \sum_{(C,D) \in \mathcal{E}} \rho(C,D) = n-k-1, \\
    \sum_{C \in \mathcal{D}} \tau(C) = n-k, \\
    \forall i \in V,  \sum_{(C,D) \in \mathcal{E}} 1_{i \in (C \cap D)} \rho(C,D)  - \sum_{C \in \mathcal{D}} 1_{i \in C} \tau(C) + 1= 0, \\
    \forall C \in \mathcal{D}, \forall (C,D) \in \mathcal{E}, \rho(C,D) \leq \tau(C),\\
    \forall C \in \mathcal{D}, \tau(C) \leq \sum_{(C,D)  \in \mathcal{E}} \rho(C,D),\\
    \text{$\rho$ is in the forest polytope of $(\mathcal{D}, \mathcal{E})$}, \\
    \text{$\tau$ is in the hyperforest polytope of $(V, \mathcal{D})$.}
    \end{array} \right.
    \label{eq:primal}
\end{equation}

All constraints except the integrality constraints are convex. 
Let $\tau${\em-relaxed primal} be the partially relaxed primal optimization problem formed by relaxing only the integral constraint 
on $\tau$ in  \eq{primal}, i.e., replacing $\tau \in \{0,1\}^\mathcal{D}$ by $\tau \in [0,1]^\mathcal{D}$. Note that this is not a convex problem due to the remaining integral constraint on $\rho$, but it remains equivalent to the original problem as the following proposition shows.
\begin{proposition}
The combinatorial problem in \eq{primal}  and the $\tau$-relaxed primal problem are equivalent.
\end{proposition}
\begin{proof}
Let us assume $(\tau^*, \rho^*)$ be a feasible solution for the relaxed primal with $0 < \tau^*(C) < 1$ for some $C \in \mathcal{D}$. 
The edge constraint in \eq{constraint4} ensures that there are no incident edges on $C$ selected by 
$\rho^*$ (as $\rho^*$ is integral). This violates the clique constraint in \eq{constraint5}. Therefore,  
the feasible solutions of relaxed primal are integral. Hence the optimal solutions of the primal and 
the relaxed primal are identical.
\end{proof}

The {\em convex relaxation} for the primal optimization problem formed by relaxing the integral constraint on both $\tau$
and $\rho$ can now be defined as 
\begin{equation}
    \min \mathcal{P}(\tau, \rho)\text{ subject to }  \left\{
    \begin{array}{l}
    \tau \in [0,1]^{\mathcal{D}},\\
    \rho \in [0,1]^{\mathcal{E}},\\
    \forall i \in V, \sum_{C \in \mathcal{D}}1_{i \in C} \tau(C) \geq 1, \\
    \sum_{(C,D) \in \mathcal{E}} \rho(C,D) = n-k-1, \\
    \sum_{C \in \mathcal{D}} \tau(C) = n-k, \\
    \forall i \in V,  \sum_{(C,D) \in \mathcal{E}} 1_{i \in (C \cap D)} \rho(C,D)  - \sum_{C \in \mathcal{D}} 1_{i \in C} \tau(C) + 1= 0, \\
    \forall C \in \mathcal{D}, \forall (C,D) \in \mathcal{E}, \rho(C,D) \leq \tau(C),\\
    \forall C \in \mathcal{D}, \tau(C) \leq \sum_{(C,D)  \in \mathcal{E}} \rho(C,D),\\
    \text{$\rho$ is in the forest polytope of $(\mathcal{D}, \mathcal{E})$}, \\
    \text{$\tau$ is in the hyperforest polytope of $(V, \mathcal{D})$.}
    \end{array} \right.
    \label{eq:cprimal}
\end{equation}

\section{Solving the dual problem}
\label{sec:LD}

We now show how the convex problem may be minimized in polynomial time. Among the constraints of our convex problem in \eq{primal}, some are simple linear constraints, some are complex constraints depending on the forest and hyperforest polytopes defined in \mysec{relax}. We will define a dual optimization problem by introducing the least possible number of Lagrange multipliers (a.k.a.~dual variables) \cite{Ber99} so that the dual function (and a supergradient) may be computed and maximized efficiently. We introduce the following dual variables:

\begin{list}{\labelitemi}{\leftmargin=1.1em}
   \addtolength{\itemsep}{-.45\baselineskip}

\item[--] Set cover constraints in \eq{constraint1}: $\gamma \in \rb_+^V$.
\item[--] Running intersection property in \eq{constraint3}: $\mu \in \rb^V$.
\item[--] Edge constraints in \eq{constraint4}:  $\lambda \in \mathbb{R}_+^{ 2\mathcal{E}}$.
\item[--] Clique constraints in \eq{constraint5}: $\eta \in \mathbb{R}_+^\mathcal{D}$.
\end{list}
Therefore, the dual variables are $(\gamma, \mu, \lambda, \eta)$. Let $\mathcal{L}(\tau, \rho, \gamma, \mu, \lambda, \eta)$ be the Lagrangian relating the primal and dual variables. It is derived from the primal 
cost function defined in \eq{primalCost} along with the covering constraint, running intersection property, the edge and the clique constraints
defined in \eq{constraint1}, \eq{constraint3}, \eq{constraint4} and \eq{constraint5} respectively. The Lagrangian can be computed from the dual
variables $(\gamma, \mu, \lambda, \eta)$ as follows:
\begin{eqnarray}
&&\mathcal{L}(\tau, \rho, \gamma, \mu, \lambda, \eta) \nonumber \\
\nonumber \\
& = &\sum_{C \in \mathcal{D}} H(C) \tau(C) - \sum_{(C,D) \in \mathcal{E}} H(C \cap D) \rho(C,D) \nonumber \\
& + &\sum_{i \in V} \gamma_i \bigg(1 - \sum_{C \in \mathcal{D}} 1_{i \in C} \tau(C)\bigg) + \sum_{i \in V} \mu_i \bigg( \sum_{(C,D) \in \mathcal{E}} 1_{i \in (C \cap D)} \rho(C,D) - \sum_{C \in \mathcal{D}} 1_{i \in C} \tau(C) + 1\bigg)   \nonumber \\
& + &\sum_{C \in \mathcal{D}} \sum_{(C,D) \in \mathcal{E}}\lambda_{CD}\bigg( \rho(C,D) - \tau(C)\bigg) + \sum_{C \in \mathcal{D}} \eta_C \bigg( \tau(C) - \sum_{(C,D) \in \mathcal{E}} \rho(C,D)\bigg)\nonumber \\
& = &\sum_{C \in \mathcal{D}}\bigg( H(C) - \sum_{i \in C}(\mu_i + \gamma_i) - \sum_{(C,D) \in \mathcal{E}} \lambda_{CD} + \eta_C\bigg) \tau(C) \nonumber \\
& - &\sum_{(C,D) \in \mathcal{E}}\bigg( H(C \cap D) - \sum_{i \in (C \cap D)}\mu_i - \lambda_{CD} - \lambda_{DC} + \eta_C + \eta_D\bigg)\rho(C,D)  + \sum_{i \in V} (\mu_i + \gamma_i), \nonumber \\
\label{eq:lagrangian}
\end{eqnarray}
with the following {\em dual constraints} on the Lagrange multipliers 
\begin{eqnarray}
\forall i \in V  ,&&  \gamma_i \geq 0, \nonumber \\
\forall C \in \mathcal{D}  ,&  \forall (C,D) \in \mathcal{E},& \lambda_{CD} \geq 0, \nonumber \\
\forall C \in \mathcal{D}  ,&&  \eta_C \geq 0. \label{eq:dualconstraint}
\end{eqnarray} 

We can now derive a dual optimization problem with $\mathcal{Q}(\gamma, \mu,\lambda,\eta)$ represent the dual cost function,
which can be derived from the Lagrangian in \eq{lagrangian}. We use the  the number of edges constraint, the number of cliques constraint, tree constraint and hyperforest constraint given by  
\eq{constraint6}, \eq{constraint7}, \eq{constraint2C} and \eq{constraintww2C} respectively in deriving the dual as follows:

\begin{eqnarray}
&&\mathcal{Q}(\gamma, \mu, \lambda, \eta) \nonumber \\
\nonumber \\
& = &\inf_{\substack{\tau(C) \in [0,1]^{\mathcal{D}}\\ \sum_{C \in \mathcal{D}} \tau(C) = n-k \\ \tau \in \text{ hyperforest polytope of }(V, \mathcal{D})}} \bigg( H(C) - \sum_{i \in C} (\mu_i + \gamma_i) - \sum_{(C,D) \in \mathcal{E}} \lambda_{CD}+ \eta_C\bigg) \tau(C) \nonumber \\
                         & - &\sup_{\substack{\rho \in [0,1]^{\mathcal{E}} \\ \sum_{(C,D) \in \mathcal{E}} \rho(C,D) = n-k-1 \\ \rho \in \text{forest polytope of }(\mathcal{D},\mathcal{E})}}  \sum_{(C,D) \in \mathcal{E}} \bigg( H(C \cap D) - \sum_{i \in (C \cap D)} \mu_i - \lambda_{CD} - \lambda_{DC} + \eta_C + \eta_D \bigg) \rho(C,D) \nonumber \\
                         & + & \sum_{i \in V} (\mu_i + \gamma_i). \label{eq:dualcost1}
\end{eqnarray}

 It is decomposed in three parts defined  in \eq{MSC}, \eq{MWF} and \eq{dual3} respectively :
\begin{equation}
\mathcal{Q}(\gamma, \mu, \lambda, \eta) =  q_1(\gamma, \mu, \lambda, \eta) + q_2(\gamma, \mu, \lambda, \eta) + q_3(\gamma, \mu, \lambda, \eta), 
\end{equation}
where 
\begin{eqnarray}
\!\!\!\!  q_1(\gamma, \mu, \lambda, \eta) \!\!\!\! & = &\inf_{\substack{\tau(C) \in [0,1]^{\mathcal{D}}\\ \sum_{C \in \mathcal{D}} \tau(C) = n-k \\ 
\tau \in \text{ hyperforest polytope of }(V, \mathcal{D})}
} \sum_{C \in \mathcal{D}}  \bigg( H(C) - \sum_{i \in C} (\mu_i + \gamma_i) - \sum_{(C,D) \in \mathcal{E}} \lambda_{CD}+ \eta_C\bigg) \tau(C). \label{eq:MSC} \\
\!\!\!\!  q_2(\gamma, \mu, \lambda, \eta)  \!\!\!\!& = & -  \!\!\!\!\!\!\!\!\!\!\!\!\!\!\!\!\!\sup_{\substack{\rho \in [0,1]^{\mathcal{E}} \\ \sum_{(C,D) \in \mathcal{E}} \rho(C,D) = n-k-1 \\ \rho \in \text{ forest  polytope of }(\mathcal{D},\mathcal{E})}}\!\sum_{(C,D) \in \mathcal{E}}\!\! \bigg( H(C \cap D) - \!\!\!\sum_{i \in (C \cap D)}  \mu_i - \lambda_{CD} - \lambda_{DC} + \eta_C + \eta_D \bigg) \rho(C,D). \nonumber \\
\label{eq:MWF}\\
\!\!\!\!  q_3(\gamma, \mu, \lambda, \eta)  \!\!\!\! & = & \sum_{i \in V} (\mu_i + \gamma_i). \label{eq:dual3}
\end{eqnarray}

Therefore, the dual optimization problem using the dual cost function defined in \eq{dualcost1} and the dual constraints defined in 
\eq{dualconstraint} is given by
\begin{equation}
    \max \mathcal{Q}(\gamma, \mu, \lambda, \eta)\text{ subject to }  \left\{
    \begin{array}{l}
    \forall i \in V, \gamma_i \geq 0, \\
    \forall C \in \mathcal{D}, \forall (C,D) \in \mathcal{E}, \lambda_{CD} \geq 0, \\
    \forall C \in \mathcal{D}, \eta_C \geq 0. \\
    \end{array} \right.
    \label{eq:dual}
\end{equation}

The dual functions  $q_1(\gamma, \mu, \lambda, \eta)$  and $q_2(\gamma, \mu, \lambda, \eta)$ may be computed using the greedy algorithms defined in \mysec{matroid1} and \mysec{matroid2};
$q_1$ can be
evaluated in $O(r \log(r))$, where $r$ is the cardinality of the space of cliques, $\mathcal{D}$, i.e., ${n \choose k+1}$ and $q_2$ can be evaluated in $O(m \log(m))$, where $m$ is the cardinality of
feasible edges, $\mathcal{E}$,  i.e., ${n \choose k+2}.{k+2 \choose 2}$. This complexity is due to sorting the edges and hyperedges based on their weights. This leads to an overall complexity of $O(k^3 n^{k+2} \log n)$ per iteration of the projected supergradient method which we now present.

\begin{algorithm*}[htb!]
\caption{Projected Supergradient}
\label{alg:projSubgrad}
\begin{algorithmic}
    \STATE {\bfseries Input:} clique and edge entropies H, step-size constant $a$ and number of iterations T
    \STATE {\bfseries Output:} sequence of clique and edge selections over iterations $(\tau^t,\rho^t)$
    \STATE Initialize $\gamma^0=0$, $\mu^0 = 0$, $\lambda^0 = 0$, $\eta^0 = 0$
    \FOR {$t=0$ {\bfseries to} $T$}
    \STATE {\bfseries solve} \eq{MSC} and evaluate $q_1(\gamma^t, \mu^t, \lambda^t, \eta^t)$ to obtain $\tau^t$
    \STATE {\bfseries solve} \eq{MWF} and evaluate $q_2(\gamma^t, \mu^t, \lambda^t, \eta^t)$ to obtain $\rho^t$
   \STATE {\bfseries update} dual variables, $(\gamma^{t+1}, \mu^{t+1}, \lambda^{t+1}, \eta^{t+1})$ using supergradients and stepsize: $\alpha_t = \frac{a}{{\sqrt t}}$\\
    $\gamma^{t+1}_i     = \Big[\gamma^t_i + \alpha_t \Big(1 - \sum_{C \in \mathcal{D}} 1_{i \in C} \tau^t(C)\Big)\Big]^+ $\\
    $\mu^{t+1}_i        = \mu^t_i + \alpha_t \Big(\sum_{(C,D) \in \mathcal{E}} 1_{i \in (C \cap D)} \rho^t(C,D) - \sum_{C \in \mathcal{D}} 1_{i \in C} \tau^t(C) + 1 \Big)$\\
    $\lambda^{t+1}_{CD} = \Big[\lambda^{t}_{CD} + \alpha_t \Big(\tau^t(C) - \rho^t(C,D)\Big)\Big]^+$\\
    $\eta^{t+1}_{C}     = \Big[\eta^t_C + \alpha_t \Big(\sum_{(C,D) \in \mathcal{E}} \rho^t(C,D) - \tau^t(C)\Big)\Big]^+ $
    \ENDFOR
\end{algorithmic}
\end{algorithm*}

\paragraph{Projected supergradient ascent.} The dual optimization problem defined by maximizing $Q(\gamma, \mu, \lambda, \eta)$ can be solved using the projected supergradient method. 
In each iteration  
$t$ of the algorithm, the dual cost function, $Q(\gamma^t, \mu^t, \lambda^t, \eta^t)$, is evaluated through estimation of $q_1$ and $q_2$ by solving \eq{MSC} and \eq{MWF} respectively. 
In the process of solving these equations, the corresponding primal variables $(\tau^t, \rho^t)$ are also estimated and allows the computations of the supergradients of $Q$ (i.e., opposites of  subgradients of $-Q$)~\cite{Ber99}. As shown in Algorithm~\ref{alg:projSubgrad}, a step is made toward the direction of the supergradient and projection onto the positive orthant is performed for dual variables that are constrained to be nonnegative. With step sizes $\alpha_t$ proportional to $1/\sqrt{t}$, this algorithm is known to converge to a dual optimal solution~\cite{nedic2009} at rate  $1/\sqrt{t}$. Moreover, the average of all visited primal variables, i.e., after $t$ steps, $({\hat \tau_t}, {\hat \rho_t}) = \frac{1}{t} \sum_{u=0}^t (\tau^u, \rho^u)$ is known to be approximately primal-feasible (i.e., it satisfies all the linear constraints that were dualized up to a small constant that is also going to zero at rate $1/\sqrt{t}$). The convergence to primal feasibility is illustrated in
\myfig{results}(a), where, on one of the synthetic examples from \mysec{results}, the different constraint violations. Note that these are not the number of each of these constraints violated but the maximum value by
which they are violated. It can be observed that the constraint violations reduce to zero over iterations.

\begin{proposition}
If $k=1$, the convex relaxation in \eq{cprimal} is equivalent to \eq{primal}.
\end{proposition}
\begin{proof}
If $k=1$, all the cliques in the clique space contain only 2 vertices, i.e., $\forall C \in \mathcal{D}, |C| = 2$ 
and the number of elements in the feasible edges is only 1, i.e., $\forall (C,D) \in \mathcal{E}, |C \cap D| = 1$.

Solving the convex relaxation defined in \eq{cprimal} is equivalent to solving the dual defined in \eq{dual}.
On solving the dual variables, the optimal dual solution is given by
\begin{equation}
\begin{array}{l}
\forall i \in V, \mu_i = H(\{i\}), \\
\forall i \in V, \gamma_i = 0, \\
\forall C \in \mathcal{D}, \forall (C,D) \in \mathcal{E}, \lambda_{CD} = 0, \\
\forall C \in \mathcal{D}, \eta_C = 0, \\
\end{array}
\end{equation}
where $H(\{i\}) = - \hat{p}_i(x_i)\log(\hat{p}_i(x_i))$.

The optimal solution to the dual problem is given by
\begin{eqnarray}
\mathcal{Q}^*(\gamma, \mu, \lambda, \eta) & = &\inf_{\substack{\tau(C) \in [0,1]^{\mathcal{D}}\\ \sum_{C \in \mathcal{D}} \tau(C) = n-k \\ 
\tau \in \text{ hyperforest polytope of }(V, \mathcal{D})}} \sum_{C \in \mathcal{D}}  \bigg( H(C) - \sum_{i \in C} H(\{i\})\bigg) \tau(C) + \sum_{i \in V} H(\{i\}) \nonumber \\
                                            & = & \inf_{\substack{\tau(C) \in [0,1]^{\mathcal{D}}\\ \sum_{C \in \mathcal{D}} \tau(C) = n-k \\ 
\tau \in \text{ hyperforest polytope of }(V, \mathcal{D})}} -I(C).\tau(C) + \sum_{i \in V} H(\{i\}), \label{eq:dualoptimal} 
\end{eqnarray}

where $\forall C \in \mathcal{D}, I(C) = \sum_{i \in C} H(\{i\}) - H(C)$, which defines the mutual information of the elements in the clique, i.e., an edge if $k=1$. The constraints in \eq{dualoptimal} define a spanning tree polytope~\cite{schrijver2004combinatorial} and the optimal solution is a maximal information spanning tree, which is given by Chow-Liu trees~\cite{Chow68approximatingdiscrete}. They also form the optimal solution to the non-convex primal optimization defined in \eq{primal}.
\end{proof}

\begin{algorithm*}[htb!]
\caption{Approximate Greedy Primal Solution}
\label{alg:projPrimalFeasible}

\begin{algorithmic}
    \STATE {\bfseries Input:} primal infeasible sequence $\tau^t$ for \myalg{projSubgrad}, treewidth $k$, number of Vertices $n$, set of cliques $\mathcal{D}$ and integer $m$ such that $0 <  m \leq T$
    \STATE {\bfseries Output:} approximate discrete primal feasible solution $\tau_m$ after $m$ iterations of \myalg{projSubgrad}
    \STATE Initialize  Adjacency Matrix $Adj = zeros(n,n)$, ${\hat \tau_m} = \frac{1}{m} \sum_{t=0}^m \tau^t$ and $\tau_m = zeros(length({\hat \tau_m}))$
    \STATE $order$ = Sorted indices in the descending order ${\hat \tau_m}$
    \REPEAT
    \STATE Initialize $decomposable = false$, $treewidth = 0$, $numConnectedComponents = 0$, $i=1$
    \STATE {\bfseries update} $TestAdj$ = AddClique($Adj$, $\mathcal{D}(order(i))$)
    \STATE {\bfseries update} $[decomposable, treewidth]$ = checkGraphDecomposability($TestAdj$)
    \IF{$decomposable$ = true {\bfseries and} $treewidth \leq k$}
    \STATE {\bfseries update} $Adj = TestAdj$
    \STATE {\bfseries update} $\tau_m(Order(i)) = 1$
    \ENDIF
    \STATE $[numConnectedComponents]$  = getNumberConnectedComponents($TestAdj$)
    \STATE {\bfseries update} $i = i + 1$
    \UNTIL{$decomposable = true$, $treewidth=k$, $numConnectedComponents=1$, $i$ = length($order$)}
\end{algorithmic}
\end{algorithm*}

\paragraph{Approximate Greedy Primal Solution.} We describe an algorithm to project from the average of a sequence of fractional primary 
infeasible solutions, estimated during the iterations of projective supergradient, to an integral primary feasible solution. 
``AddClique" adds all the edges of a clique to the adjacency matrix. ``checkGraphDecomposability" checks if the maximal 
cardinality search  is a perfect elimination ordering. For decomposable graphs the maximal cardinality search yields 
a perfect elimination ordering~\cite{Golumbic:2004}. We refer to this as {\em decomposability test} in this paper.
``getNumberConnectedComponents" gives the number of connected components in the graph using breadth-first search. 
Note that the projection only uses the average clique selection function, ${\hat \tau_m}$, to obtain the primary 
feasible solutions, $\tau_m$. The corresponding edge selection, $\rho_m$, can be estimated from clique selection, 
$\tau_m$, by selecting the edges between consecutive cliques of the perfect sequence of selected cliques~\cite{Lauritzen}. 
The time complexity of the projection algorithm is $O(n^{k+2})$. This is due to decomposability test with run time 
complexity $O(n^{k+1})$, that is performed on adding $O(n)$ cliques.

\section{Experiments and Results}
\label{sec:results}
In this section, we show the performance of the proposed algorithm on synthetic datasets and classical benchmarks.

\paragraph{Decomposable covariance matrices.}
In order to easily generate controllable distributions with entropies which are easy to compute, we use several decomposable graphs and we consider a Gaussian vector with covariance matrix $\Sigma$, generated as follows: 
\begin{itemize}
\item sample a matrix $Z$ of dimensions $n \times d'$ with entries uniform in $[0,1]$ and consider the matrix
\begin{eqnarray}
\Sigma' = \frac{d}{d'} Z Z^\top + (1 - \frac{d}{d'})I,
\label{eq:pdm}
\end{eqnarray}
where $Z$ is a random matrix of dimensions $n \times d'$, $I$ is the $n$-dimensional identity matrix and $d$ is a parameter to  determine the
correlations between the nodes of the graph, which takes values in $\{0, d'\}$. In our experiments, we choose $d'$ to be $128$. We have tight 
correlations between the nodes with higher values of $d$.

\item normalize $\Sigma'$ to unit diagonal, and 
\item  The normalized random positive definite covariance matrix, $\Sigma'$, is projected onto a decomposable graph $G$ 
as follows:
\begin{eqnarray}
(\Sigma)^{-1} = \sum_{C \in \mathcal{C}(G)} [(\Sigma'_C)^{-1}]_n - \sum_{(C,D) \in \mathcal{T}(G)} [(\Sigma'_{C \cap D})^{-1}]_n,
\label{eq:projection}
\end{eqnarray}
where the operator $[(\Sigma'_X)^{-1}]_n$ gives an $n \times n$ matrix whose columns and rows representing the set $X \subseteq V$ are filled by $(\Sigma'_X)^{-1}$ and the
rest of the elements of the matrix are filled with {\em zeroes}. The matrix, $\Sigma$, thus generated represents the
covariance matrix of a multivariate Gaussian on a decomposable graph, $G$.
\end{itemize}

The projection ensures the following relationship between the random positive definite matrix, $\Sigma'$ and the projected covariance matrix $\Sigma$:
\begin{eqnarray}
\Sigma(i,j)         & = & \Sigma'(i,j)  \text{ if } A(i,j) = 1 \text{ or } i = j, \nonumber \\
{\Sigma}^{-1}(i, j) & = & 0             \text{ if } A(i,j) = 0.
\label{eq:covariance}
\end{eqnarray}
 where $A$ is the adjacency matrix of the decomposable graph $G$ onto which $\Sigma'$ was projected.

The entropy of a multivariate Gaussian with a covariance matrix, $\Sigma$, is given by 
$\frac{1}{2}\log(2\pi e)^n|\Sigma|$, where $|\Sigma|$ denotes the determinant of the covariance matrix. 
However, for Gaussian distribution that is factored in $G \in \mathcal{G}$:
\begin{equation}
|\Sigma| = \frac{\prod_{C \in \mathcal{C}(G)} |\Sigma_{C}|}{\prod_{((C,D) \in \mathcal{T}(G)} |\Sigma_{C \cap D}|},
\end{equation}
where $\Sigma_{X}$ is the sub-matrix of the covariance matrix whose rows and columns belong to the set $X \subseteq V$. Therefore, for any multivariate decomposable Gaussian graphical model, $G$:
\begin{eqnarray}
H(G) & = & \frac{1}{2}\log((2\pi e)^n|\Sigma|) \nonumber \\
     & = & \frac{1}{2}(\sum_{C \in \mathcal{C}(G)} \log((2\pi e)^n|\Sigma_{C}|) - \sum_{(C,D) \in \mathcal{T}(G)} \log((2\pi e)^n|\Sigma_{C \cap D}|)) \nonumber \\
     & = & \sum_{C \in \mathcal{C}(G)} H(C) - \sum_{(C,D) \in \mathcal{T}(G)} H(C \cap D).
\label{eq:entropyCov}
\end{eqnarray}
Note that the entropy of any graph, G, is independent of the mean of the normal distribution, hence we consider only the covariance matrix.

We use the graph structures representing a {\em chain 
junction tree} as in \myfig{JTGT}-(a) and a {\em star junction tree} as in \myfig{JTGT}-(b) to analyze the performance of our 
algorithm for decomposable covariance matrices generated with different correlations.  

Table~\ref{table:chainJT} and Table~\ref{table:starJT} show the performance of our algorithm on these two graphs. 
Decomposable covariance matrices are generated as above with 
different values of the correlation parameter $d$ (all averaged over ten different random covariance matrices). We show the difference between 
the cost function in \eq{primalCost} and the optimal entropy, i.e., the one of the
actual structure represented by the covariance matrices. The differences in the table are multiplied by $10^3$ for brevity.

The first column $\Delta$Dual represents the optimal value of our convex relaxation (obtained from the dual function), while the second column 
$\Delta$Dual$^r$ represents the optimal value by replacing the hyperforest constraint by the simply $\tau \in [0,1]^\mathcal{D}$. 
We can see from the two tables, that the two values are strictly negative (i.e., we indeed have a relaxation) and that the hyperforest 
constraint is key to obtaining tighter relaxations. Note that the associated solutions are only fractional.

The third column $\Delta$Primal represents the cost function obtained by projection of the optimal fractional solution of the hyperforest constraint,
using {\em Approximate Greedy Primal Solution} algorithm. The fourth column $\Delta$Primal$^r$ represents the cost function obtained by
projecting the optimal fractional solution of the hypercube constraint, i.e., the corresponding primal feasible solution related to $\Delta$Dual$^r$. 
They are compared to a simple greedy algorithm in the fifth column that sorts all mutual information and keep adding the cliques with largest 
mutual information as long as decomposability is maintained. Although the relaxation is not tight, our rounding scheme leads empirically to 
the optimal solution when the correlations are strong enough (i.e., large values of $d$) and outperform the simple greedy algorithm.

\begin{figure*}[htb]
\begin{center}
\begin{tabular}{cc}
\includegraphics[width=0.32\textwidth]{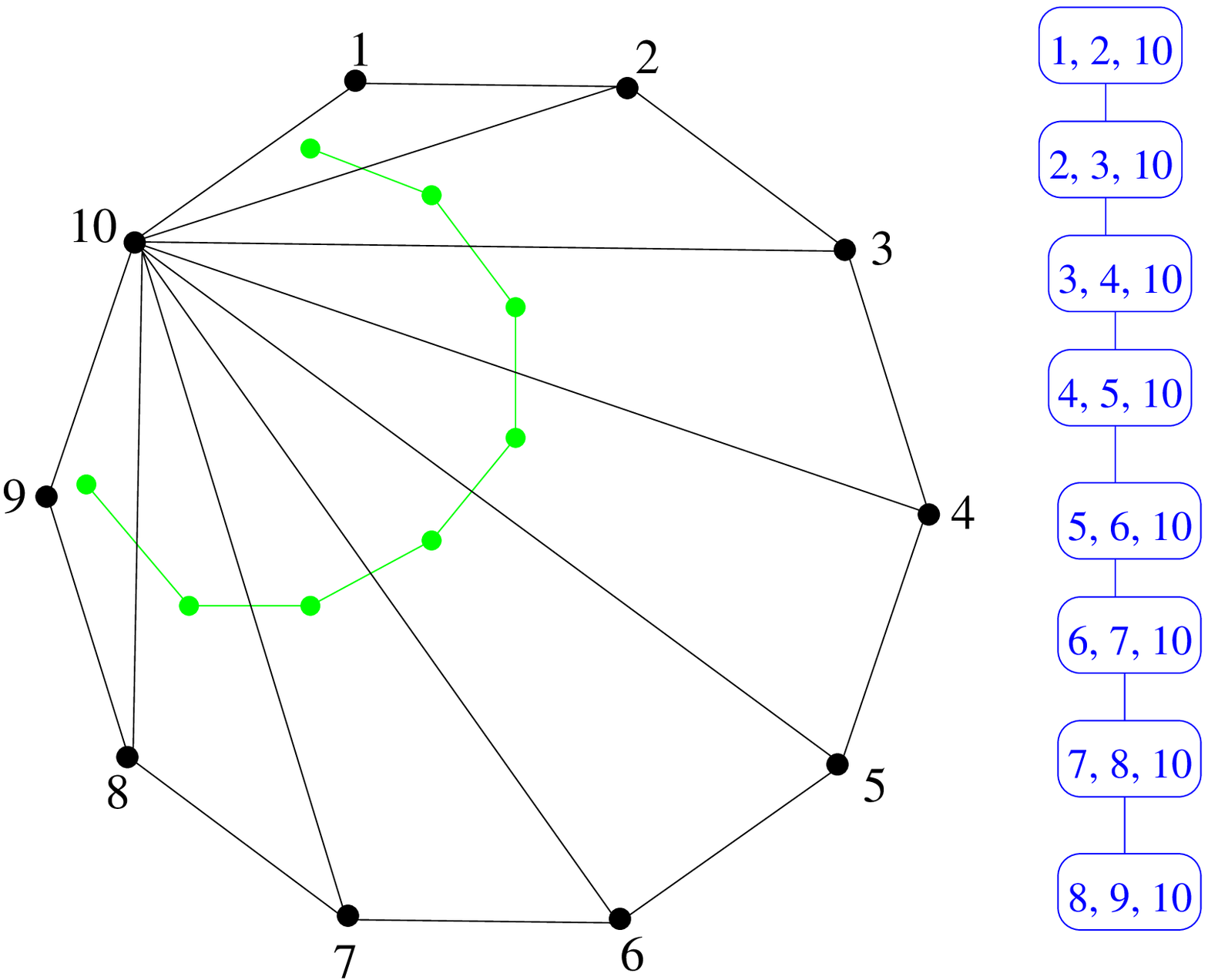}&
\includegraphics[width=0.55\textwidth]{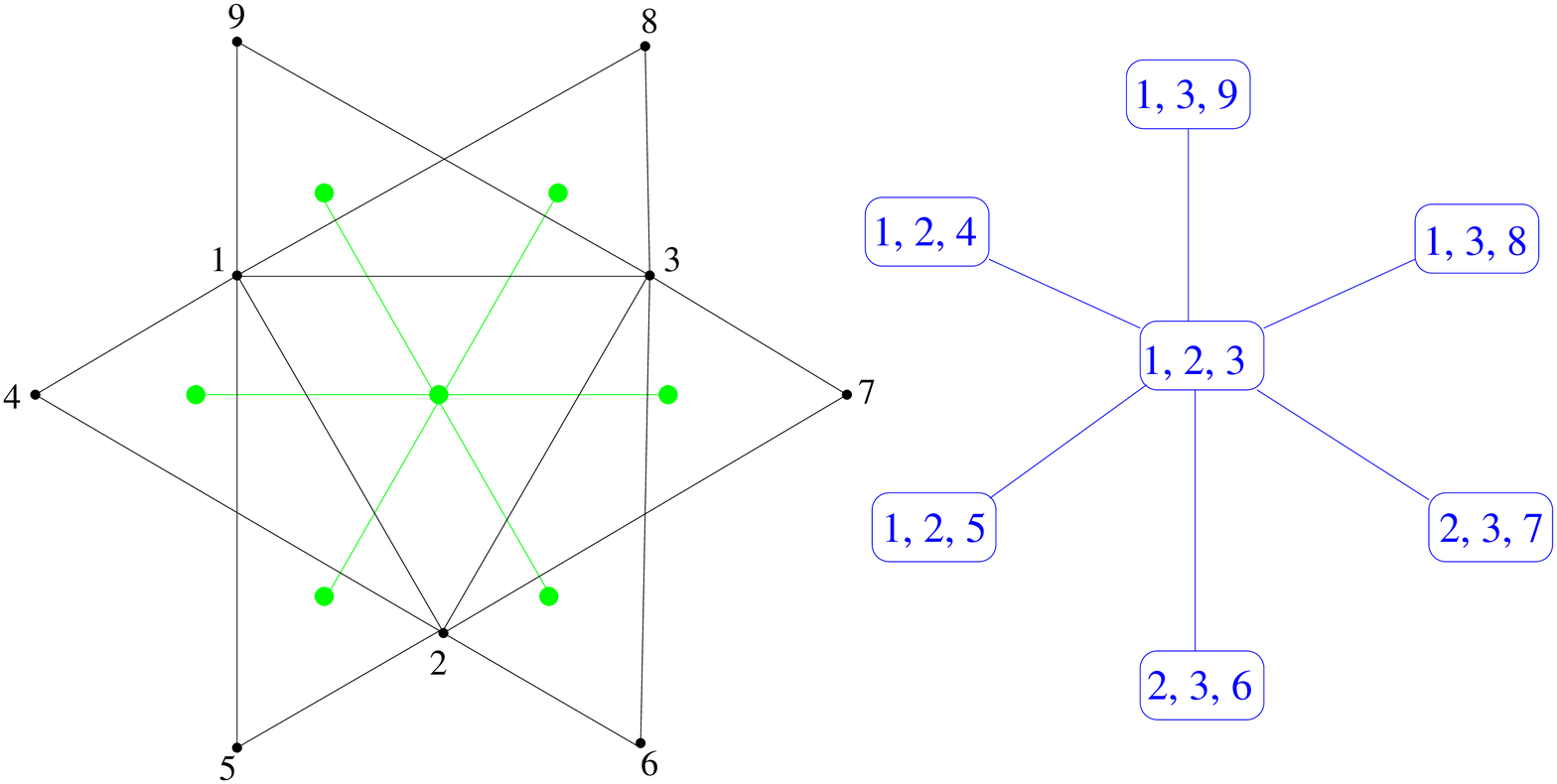}\\
(a) & (b)
\end{tabular}
\end{center}

\caption{Graph representing (a) chain junction tree, (b) star junction tree, with an embedded junction tree in green and its junction tree representation in blue.}
\label{fig:JTGT}
\end{figure*}

\begin{figure*}[htb]
\begin{center}
\begin{tabular}{ccc}
\includegraphics[width=0.32\textwidth]{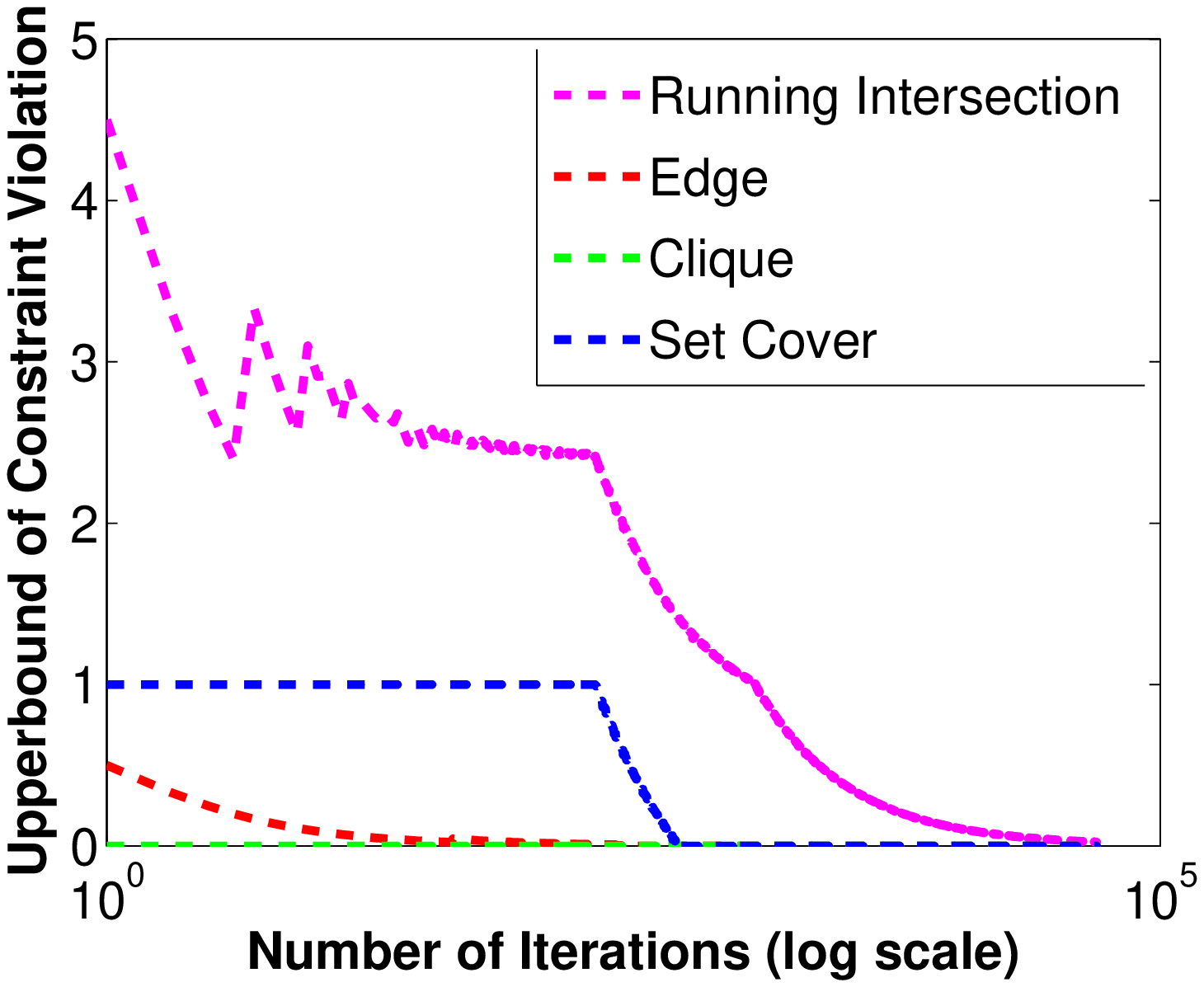}&
\includegraphics[width=0.32\textwidth]{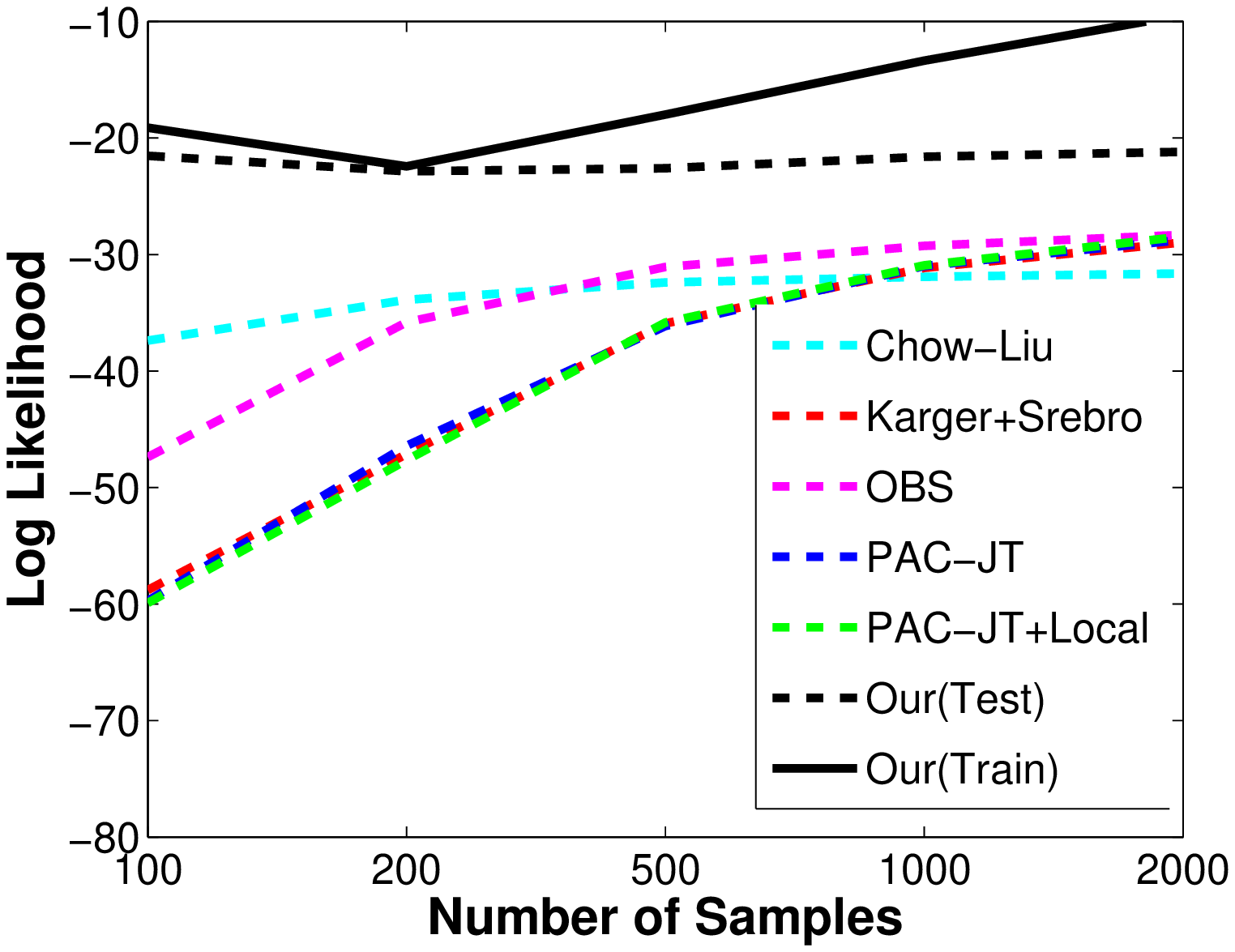}&
\includegraphics[width=0.32\textwidth]{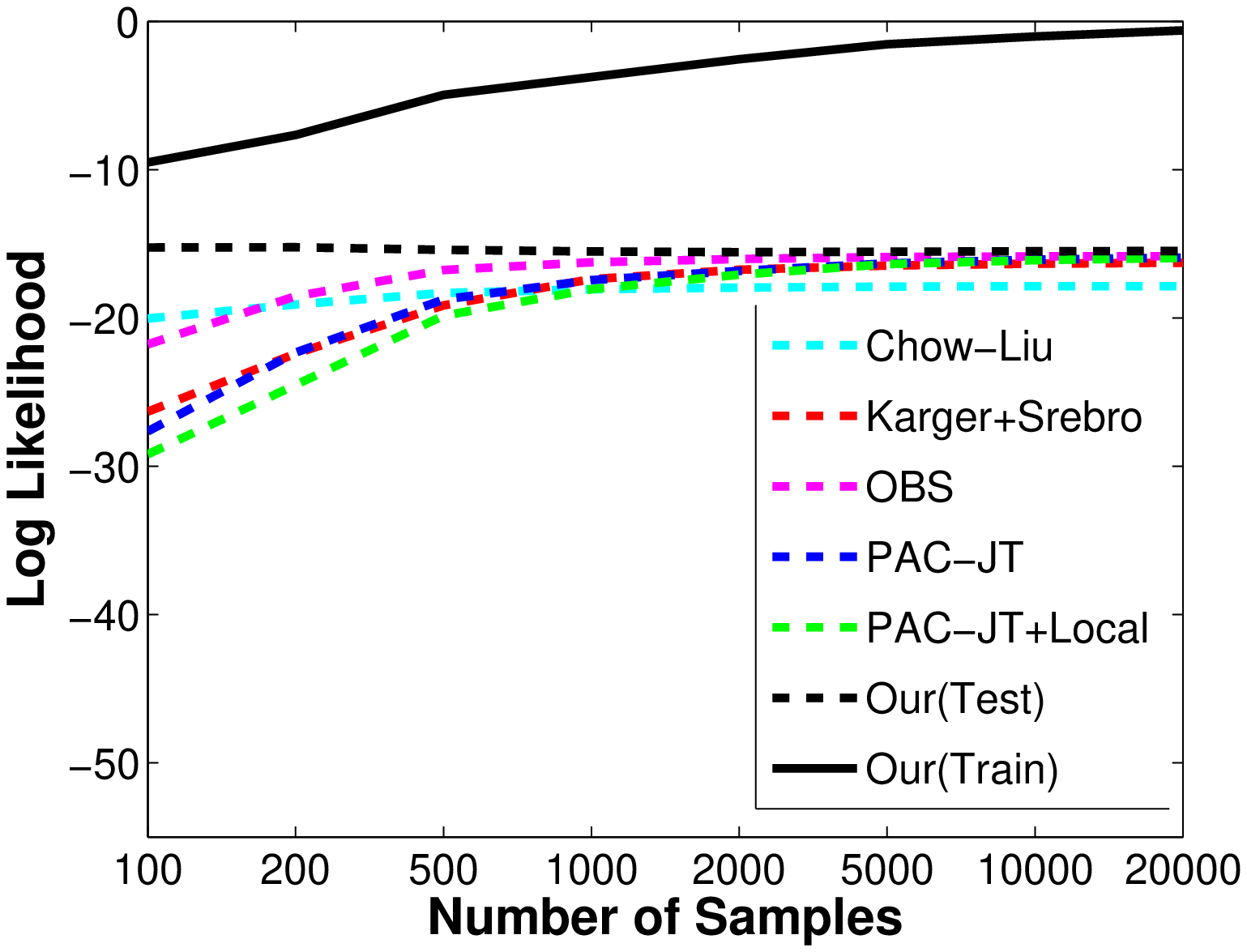}\\
(a) & (b) & (c)
\end{tabular}
\end{center}

\caption{Left: (a) Upper bound of constraint violations for d=2 and a chain junction tree. Right: Log likelihood of the structures learnt using various algorithms on (b) TRAFFIC and (c) ALARM datasets with $k=3$ except Chow-Liu $(k=1)$.}
\label{fig:results}
\end{figure*}

\begin{table}[!htb]
\caption{Performance on chain junction trees. See text for details. }

\begin{center}
\begin{tabular}{|r|c|c|c|c|c|}
\hline
d     & $\Delta$Dual & $\Delta$Dual$^r$ & $\Delta$Primal & $\Delta$Primal$^r$ & $\Delta$Greedy    \\ 
\hline                                                   
1    & -0.7$\pm$0.1 & -32.7$\pm$16.4 &  0.2$\pm$0.1      &    0.4$\pm$0.1     &  0.2$\pm$0.1 \\
2    & -0.4$\pm$0.1 & -23.4$\pm$9.6  &      0            &    0.3$\pm$0.2     &  0.5$\pm$0.2 \\
4    & -1.1$\pm$0.1 & -31.2$\pm$9.7  &      0            &    0.3$\pm$0.1     &  1.9$\pm$0.3 \\
8    & -0.6$\pm$0.1 & -23.9$\pm$9.8  &      0            &    0.2$\pm$0.1     &  7.9$\pm$0.3 \\
16   & -1.9$\pm$0.2 & - 3.4$\pm$2.7  &      0            &         0          &  25.6$\pm$1.2 \\
32   & -3.9$\pm$0.5 & - 3.2$\pm$0.3  &      0            &         0          &  57.3$\pm$1.5 \\
\hline
\end{tabular}

\end{center}
\label{table:chainJT}
\end{table}

\begin{table}[!htb]
\caption{Performance on star junction trees. See text for details. }

\begin{center}
\begin{tabular}{|r|c|c|c|c|c|}
\hline
d    & $\Delta$Dual  & $\Delta$Dual$^r$   & $\Delta$Primal & $\Delta$Primal$^r$ &  $\Delta$Greedy    \\ 
\hline                                                                             
1     & -0.8$\pm$0.1  & -31.4$\pm$13.4   &   0.2$\pm$0.1  &     0.5$\pm$0.1     &  0.9$\pm$0.1 \\
2     & -0.5$\pm$0.2  & -26.6$\pm$13.3   &      0         &     0.4$\pm$0.1     &  0.4$\pm$0.3 \\
4     & -0.3$\pm$0.0  & -16.6$\pm$4.1    &      0         &     0.2$\pm$0.1     &  1.7$\pm$0.2 \\
8     & -0.4$\pm$0.0  & -16.0$\pm$9.6    &      0         &          0          &  6.9$\pm$0.3 \\
16    & -1.2$\pm$0.5  & -3.1$\pm$0.3     &      0         &          0          &  26.3$\pm$1.5 \\
32    & -6.8$\pm$0.4  & -8.5$\pm$1.2     &      0         &          0          &  58.3$\pm$1.9 \\
\hline
\end{tabular}

\end{center}
\label{table:starJT}
\end{table}

\paragraph{Performance Comparison.} We compare the quality of the graph structures learned by the proposed algorithm with the ones 
produced by Ordering Based Search (OBS)~\cite{Teyssier:2005}, the combinatorial optimization algorithm proposed by Karger and Srebro 
(Karger+Srebro)~\cite{Karger:2001}, the Chow-Liu trees (Chow-Liu)~\cite{Chow68approximatingdiscrete} and different variations 
of PAC-learning based algorithms (PAC-JT, PAC-JT+local)~\cite{Chechetka:2007}. We use a real-world dataset, TRAFFIC~\cite{traffic} and 
an artificial dataset, ALARM~\cite{alarm} to compare the performances of these algorithms. 

This {ALARM} dataset was sampled from a known Bayesian network~\cite{alarm} of 37 nodes, which has a treewidth equal to 4. We learn 
an approximate decomposable graph of treewidth 3. The {TRAFFIC} dataset is the traffic flow information every 5 minutes 
for a month at 8000 locations in California~\cite{traffic}. The traffic flow information is collected at 32 locations in 
San Francisco Bay area and the values are discretized into 4 bins. We learn an approximate decomposable graph of treewidth 3 
using our approach.  Empirical entropies are computed from the generated samples of each data set and we infer the underlying
structure from them using our algorithm. \myfig{results}(b) and \myfig{results}(c) show the log-likelihoods of structures learnt using various algorithms 
on Traffic and Alarm datasets respectively. Note that the performance is better with higher values as we compare log-likelihoods. 
These figures illustrate the gains of the convex approach over the earlier non-convex approaches.

\section{Conclusion and Future Work}
\label{sec:conclusion}

In this paper, we have provided a convex relaxation to the problem of finding the maximum likelihood decomposable graph with bounded treewidth, with a polynomial-time optimization algorithm, which empirically outperforms previously proposed algorithms.
We are currently exploring two avenues for improvements: (a) design sufficient conditions for tightness of our relaxation, following the recent literature on relaxation of variable selection problems~\cite{candes2005decoding}, and (b) use heuristics to speed-up the  algorithms to allow application to larger graphs.

\paragraph{Acknowledgements}

We acknowledge support from the European Research Council grant SIERRA (project 239993), and would like to thank Anton Chechetka, Carlos Guestrin, Nathan Srebro and Percy Liang for sharing code and datasets. We would also like to thank other members of the SIERRA and WILLOW project-teams for helpful discussions.

\bibliography{convextjt}

\end{document}